\pdfoutput=1 %
\documentclass[wcp]{jmlr} %

\usepackage{url}            %
\usepackage{booktabs}       %
\usepackage{amsmath, amsfonts, amssymb} %
\usepackage{mathtools}      %
\usepackage{nicefrac}       %
\usepackage{microtype}      %
\usepackage{graphicx}       %
\graphicspath{{figures/}}   %
\usepackage{doi}            %
\usepackage{subfiles}       %
\usepackage{listings}       %
\usepackage[normalem]{ulem} %
\usepackage[]{todonotes}    %
\usepackage{orcidlink}
\usepackage{xparse}

\usepackage{xcolor}      %

\lstdefinestyle{python}{
  language=Python,
  basicstyle=\ttfamily\small,
  keywordstyle=\color{blue},
  commentstyle=\color{gray}\itshape,
  stringstyle=\color{orange},
  showstringspaces=false,
  breaklines=true,
  frame=single,
  framesep=4pt,
}

\newcommand{\Ex}{\mathbb{E}}
\newcommand{\var}{\operatorname{Var}}

\newcommand{\vv}[1]{\boldsymbol{#1}}
\newcommand{\mm}[1]{\mathrm{#1}}
\newcommand{\mmmean}[1]{\bar{\mathrm{#1}}}
\newcommand{\mmdev}[1]{\breve{\mathrm{#1}}}
\newcommand{\rv}[1]{\mathsf{#1}}
\newcommand{\vrv}[1]{\vv{\rv{#1}}}

\newcommand{\op}[1]{\mathcal{#1}}

\newcommand{\disteq}{\stackrel{\mathrm{d}}{=}}
\newcommand{\Normal}{\mathcal{N}}

\newcommand{\gvn}{\mid}
\renewcommand{\Pr}{\mathbb{P}}

\NewDocumentCommand{\Law}{o}{%
  \mu\IfValueT{#1}{_{#1}}%
}

\NewDocumentCommand{\ELaw}{o}{%
  \widehat{\mu}\IfValueT{#1}{_{#1}}%
}

\jmlrproceedings{arXiv Preprint}{Submitted to arXiv on \today}

\title[Ensemble Kalman Update]{The Ensemble Kalman Update is an Empirical Matheron Update}

\author{\Name{Dan MacKinlay\,\orcidlink{0000-0001-6077-2684}}\\
  \addr {CSIRO's Data61}
}

\begin{document}

\maketitle

\begin{abstract}
The Ensemble Kalman Filter (EnKF) is a widely used method for data assimilation in high-dimensional systems, with an ensemble update step equivalent to an empirical version of the Matheron update popular in Gaussian process regression—a connection that links half a century of data-assimilation engineering to modern path-wise GP sampling.

This paper provides a compact introduction to this simple but under-exploited connection, with necessary definitions accessible to all fields involved.

Source code is available at \url{https://github.com/danmackinlay/paper_matheron_equals_enkf}.

}

\end{abstract}

\section{Introduction}

The Ensemble Kalman Filter (EnKF)~\citep{Evensen2003Ensemble,Evensen2009Data} is a cornerstone method that evolves ensembles of state vectors through model dynamics and updates them using observational data. The Matheron update provides a sample-based method for conditioning Gaussian random variables on observations~\citep{Doucet2010Note,Wilson2020Efficiently,Wilson2021Pathwise}, well-established in geostatistics but with little-known connections to ensemble data assimilation.

We establish that the ensemble update step in the EnKF is equivalent to an empirical Matheron update by putting them on a common probabilistic footing. This connection provides an alternative foundation for the EnKF and suggests improvements by leveraging computational optimizations from both communities.

\subsection{Historical Context and Related Work}\label{sec:history}

Although the affine residual update first appeared in geostatistics under the banner of \emph{conditioning by kriging}~\citep{Chiles2018Fifty}, it was rediscovered many times:
\begin{itemize}
    \item \textbf{Optimal interpolation (OI).}  Early numerical-weather-prediction (NWP) systems wrote the analysis as
    $x_a = x_b + K(y - Hx_b)$ with a climatological Kalman gain; this is exactly the Matheron rule with fixed covariances~\citep{Hunt2007Efficient}.
    \item \textbf{Stochastic EnKF.}  \citet{Evensen2003Ensemble} replaced the static climatological covariances of OI with an empirical ensemble, creating the now-ubiquitous EnKF.
    \item \textbf{Fast conditional simulation.}  \citet{Doucet2010Note} noticed that one can obtain conditional Gaussian draws \emph{without} a Cholesky factorisation by adding a linear residual—again the same formula.
    \item \textbf{Pathwise GP sampling.}  The machine-learning community re-invented the idea as \emph{path-wise conditioning} for scalable Gaussian-process (GP) posterior draws~\citep{Wilson2020Efficiently,Wilson2021Pathwise,Borovitskiy2023Matern}.
    \item \textbf{Reservoir inversion and EnRML.}  Iterative ensemble smoothers such as EnRML exploit the very same affine map in the context of PDE-constrained Bayesian inversion~\citep{Chen2012Ensemble}.
    \item \textbf{Theory of ensemble inversion.}  Convergence proofs for ensemble inversion~\citep{Schillings2017Analysis} and extensions via transport maps~\citep{Spantini2022Coupling} all start from the linear Matheron/EnKF identity.
\end{itemize}

This equivalence forms the algebraic core of at least six research communities, turning a patchwork of field-specific heuristics into transferable technology.

\subsection{Notation}

We write random variates sans serif, $\vrv{x}$.
Equality in distribution is $\disteq$.
The law, or measure, of a random variate $\vrv{x}$ is denoted $\Law[\vrv{x}]$,
so that $\left(\Law[\vrv{x}]=\Law[\vrv{y}]\right) \Rightarrow \left(\vrv{x}\disteq\vrv{y}\right)$.
Mnemonically, samples drawn from the $\Law[\vrv{x}]$ are written with a serif $\vv{x}\sim\Law[\vrv[x]]$.
We use a hat to denote empirical estimates, e.g. \(\ELaw[\mm{X}]\) is the empirical law induced by the sample matrix \(\mm{X}\).
When there is no ambiguity we suppress the sample matrix, writing simply \(\widehat{\Law}\).
We follow standard measure notation; Appendix~\sectionref{sec:densities-please} provides a density-based translation.

\section{Matheron Update}

The Matheron update is a technique for sampling from the conditional distribution of a Gaussian random variable given observations, without explicitly computing the posterior covariance \citep{Doucet2010Note,Wilson2020Efficiently,Wilson2021Pathwise}.

\begin{lemma}[Matheron Update]
Given a jointly Gaussian vector
\begin{align}
    \begin{bmatrix} \vrv{x} \\ \vrv{y} \end{bmatrix}
    &\sim \Normal\left(\begin{bmatrix} \vv{m}_{\vrv{x}} \\ \vv{m}_{\vrv{y}} \end{bmatrix}, \begin{bmatrix} \mm{C}_{\vrv{xx}} & \mm{C}_{\vrv{xy}} \\ \mm{C}_{\vrv{yx}} & \mm{C}_{\vrv{yy}} \end{bmatrix}\right), \label{eq:joint-gaussian}
\end{align}
the conditional $\vrv{x} | \vrv{y} {=} \vv{y}^*$ is equal in distribution to
\begin{align}
    \left(\vrv{x} \gvn \vrv{y} {=} \vv{y}^*\right)
    &\disteq \vrv{x} + \mm{C}_{\vrv{xy}} \mm{C}_{\vrv{yy}}^{-1} \left( \vv{y}^* - \vrv{y} \right).
    \label{eq:matheron-update}
\end{align}
\end{lemma}

\begin{proof}
    A standard property of the Gaussian \citep[e.g.][]{Petersen2012Matrix} is that the conditional distribution of a Gaussian variate  $\vrv{x}$ given $\vrv{y} = \vv{y}^*$ defined as in \eqref{eq:joint-gaussian} is again Gaussian
    \begin{align}
        \left(\vrv{x} \gvn \vrv{y} {=} \vv{y}^*\right)
        \sim\Normal(\vv{m}_{\vrv{x}\gvn\vrv{y}}, \mm{C}_{\vrv{x}\gvn\vrv{y}})\label{eq:conditional-gaussian}
    \end{align}
    with moments
    \begin{align}
        \vv{m}_{\vrv{x}\gvn\vrv{y}}
            &=\Ex [\vrv{x} \gvn \vrv{y} {=} \vv{y}^*] \\
            &= \vv{m}_{\vrv{x}} + \mm{C}_{\vrv{xy}} \mm{C}_{\vrv{yy}}^{-1} \left( \vv{y}^* - \vv{m}_{\vrv{y}} \right), \label{eq:conditional-mean}\\
        \mm{C}_{\vrv{x}\gvn\vrv{y}}
            &= \var \left(\vrv{x} \gvn \vrv{y} {=} \vv{y}^*\right) \\
            &= \mm{C}_{\vrv{xx}} - \mm{C}_{\vrv{xy}} \mm{C}_{\vrv{yy}}^{-1} \mm{C}_{\vrv{yx}}. \label{eq:conditional-cov}
    \end{align}
Taking moments of the right hand side of \eqref{eq:matheron-update}
\begin{align}
\Ex\left[\vrv{x}+\mm{C}_{\vrv{xy}} \mm{C}_{\vrv{yy}}^{-1}(\vv{y}^*-\vrv{y})\right]
&=\vv{m}_{\vrv{x}} +\mm{C}_{\vrv{xy}}\mm{C}_{\vrv{yy}}^{-1}(\vv{y}^*-\vv{m}_{\vrv{y}})\\
&=\vv{m}_{\vrv{x}\gvn\vrv{y}}\\
\var\left[\vrv{x}+\mm{C}_{\vrv{xy}} \mm{C}_{\vrv{yy}}^{-1}(\vv{y}^*-\vrv{y})\right]
&=
    \var[\vrv{x}]+\var[\mm{C}_{\vrv{xy}} \mm{C}_{\vrv{yy}}^{-1}(\vv{y}^*-\vrv{y})] \nonumber \\
    &\hspace{2em} +\var(\vrv{x},\mm{C}_{\vrv{xy}} \mm{C}_{\vrv{yy}}^{-1}(\vv{y}^*-\vrv{y}))\nonumber \\
    &\hspace{2em} +\var(\vrv{x},\mm{C}_{\vrv{xy}} \mm{C}_{\vrv{yy}}^{-1}(\vv{y}^*-\vrv{y}))^{\top}\\
&=\mm{C}_{\vrv{x}\vrv{x}} +\mm{C}_{\vrv{xy}} \mm{C}_{\vrv{yy}}^{-1}\mm{C}_{\vrv{yy}} \mm{C}_{\vrv{yy}}^{-1}\mm{C}_{\vrv{yx}}
-  2\mm{C}_{\vrv{xy}} \mm{C}_{\vrv{yy}}^{-1}\mm{C}_{\vrv{yx}}\\
&=\mm{C}_{\vrv{x}\vrv{x}} -\mm{C}_{\vrv{xy}} \mm{C}_{\vrv{yy}}^{-1}\mm{C}_{\vrv{yx}} =\mm{C}_{\vrv{x}\gvn\vrv{y}}
\end{align}
we see that both first and second moments match.
Since $\Law\left\{\vrv{x} \gvn \vrv{y} {=} \vv{y}^*\right\}$ and
$\Normal(\vv{m}_{\vrv{x}\gvn\vrv{y}}, \mm{C}_{\vrv{x}\gvn\vrv{y}})$ are Gaussian distributions with the same moments, they define the same distribution.
\end{proof}

This \emph{pathwise} approach to conditioning Gaussian variates has gained currency in machine learning as a tool for sampling and inference in Gaussian processes \citep{Wilson2020Efficiently,Wilson2021Pathwise}, notably in generalising to challenging domains such as Riemannian manifolds~\citep{Borovitskiy2023Matern}.

\section{Kalman Filter}
We begin by recalling that in state filtering the goal is to update our estimate of the system state when a new observation becomes available. In a general filtering problem the objective is to form the posterior distribution \(\Law[\vrv{x}_{t+1} \gvn \vrv{x}_t, \vv{y}^*]\) from the prior \(\Law[\vrv{x}_{t+1} \gvn \vrv{x}_t]\) by incorporating the new information in the observation \(\vv{y}^*\).

We assume a known observation operator $\op{H}$ such that observations $\vrv{y}$ are a priori related to state $\vrv{x}$ by $\vrv{y}=\op{H}(\vrv{x})$; Thus there exists a joint law for the prior random vector
\begin{align}
    \begin{bmatrix}
        \vrv{x}\\
        \vrv{y}
    \end{bmatrix} &= \begin{bmatrix}
        \vrv{x}\\
        \op{H}(\vrv{x})
    \end{bmatrix}\label{eq:joint-law}
\end{align}
which is determined by the prior state distribution $\Law[\vrv{x}]$ and the observation operator $\op{H}$.
The  \emph {analysis} step, in state filtering parlance, is the update at time $t$ of  \(\Law[\vrv{x}_{t}]\), into the posterior \( \Law[\vrv{x}_t \gvn (\vrv{y}_t{=}\vv{y}_t^*)]\)
i.e. incorporating the likelihood of the observation $\vv{y}_t^*=\vv{y}_t$.
Although  recursive updating in $t$ is the focus of the classic Kalman filter, in this work we are concerned only with the observational update step.
Hereafter we suppress the time index $t$, and consider an individual analysis update.

Suppose that the state and observation noise are independent, and all variates are defined over a finite dimensional real vector space
$\vv{x}\in\mathbb{R}^{D_{\vrv{x}}}, \vv{y}\in \mathbb{R}^{D_{\vrv{y}}}.$
Suppose, moreover, that at the update step our prior belief about $\vrv{x}$ is Gaussian mean \(\vv{m}_{\vrv{x}}\) and covariance \(\mm{C}_{\vrv{xx}}\), that the observation noise is centred Gaussian with covariance \(\mm{R}\), and the observation operator is linear with matrix \(\mm{H}\), so that the observation is related to the state via
\[
\vv{y} = \mm{H}\,\vv{x} + \vv{\varepsilon}, \quad \vv{\varepsilon} \sim \mathcal{N}(0,\mm{R}).
\]
Then the joint distribution of the prior state and observation is Gaussian and \eqref{eq:joint-law} implies that it is
\begin{align}
\begin{bmatrix}
\vrv{x} \\
\vrv{y}
\end{bmatrix}
&\sim \mathcal{N}\!\left(
    \begin{bmatrix}
    \vv{m}_{\vrv{x}} \\
    \vv{m}_{\vrv{y}}
    \end{bmatrix},
    \begin{bmatrix}
    \mm{C}_{\vrv{xx}} & \mm{C}_{\vrv{xy}} \\
    \mm{C}_{\vrv{yx}} & \mm{C}_{\vrv{yy}}
    \end{bmatrix}
    \right)\\
&=\mathcal{N}\!\left(
    \begin{bmatrix}
    \vv{m}_{\vrv{x}} \\
    \mm{H}\,\vv{m}_{\vrv{x}}
    \end{bmatrix},
    \begin{bmatrix}
    \mm{C}_{\vrv{xx}} & \mm{C}_{\vrv{xx}}\,\mm{H}^\top \\
    \mm{H}\,\mm{C}_{\vrv{xx}} & \mm{H}\,\mm{C}_{\vrv{xx}}\,\mm{H}^\top + \mm{R}
    \end{bmatrix}
    \right)
\end{align}
When an observation \(\vv{y}^*\) is obtained, we apply the formulae
\eqref{eq:conditional-gaussian}, \eqref{eq:conditional-mean}, and \eqref{eq:conditional-cov} to calculate
\begin{align}
\left(\vrv{x} \gvn \vrv{y} {=} \vv{y}^*\right)
&\sim\Normal(\vv{m}_{\vrv{x}\gvn\vrv{y}}, \mm{C}_{\vrv{x}\gvn\vrv{y}})\\
\vv{m}_{\vrv{x}\gvn \vv{y}}
&= \vv{m}_{\vrv{x}} + \mm{K} \left(\vv{y}^* - \vv{m}_{\vrv{y}}\right),\\
&= \vv{m}_{\vrv{x}} + \mm{K} \left(\vv{y}^* - \mm{H}\,\vv{m}_{\vrv{x}}\right),\\
\mm{C}_{\vrv{x}\gvn\vrv{y}}
&= \mm{C}_{\vrv{xx}} - \mm{K}\mm{C}_{\vrv{yx}}\\
&= \mm{C}_{\vrv{xx}} - \mm{K}\,\mm{H}\,\mm{C}_{\vrv{xx}}.
\end{align}
where
\begin{align}
\mm{K}
&\coloneq \mm{C}_{\vrv{xy}}\,\mm{H}^\top \left(\mm{C}_{\vrv{yy}}\right)^{-1},\label{eq:kalman-gain}\\
&\coloneq \mm{C}_{\vrv{xx}}\,\mm{H}^\top \left(\mm{H}\,\mm{C}_{\vrv{xx}}\,\mm{H}^\top + \mm{R}\right)^{-1},
\end{align}
is the \emph{Kalman gain}.
Hereafter we consider a constant diagonal \(\mm{R}\) for simplicity, so that \(\mm{R}=\rho^2\mm{I}_{D_{\vrv{y}}}\).

\section{Ensemble Kalman Filter}

In high-dimensional or nonlinear settings, directly computing these posterior updates is often intractable.
The Ensemble Kalman Filter (EnKF) addresses this issue by representing the belief about the state empirically, via an ensemble of \(N\) state vectors sampled from the prior distribution,
\begin{align}
    \mm{X} = \begin{bmatrix} \vv{x}^{(1)} & \vv{x}^{(2)} & \cdots & \vv{x}^{(N)} \end{bmatrix},
\end{align}
and working with the empirical measure $\ELaw[\mm{X}]\approx \Law$.

Gaussian measures are specified entirely by their first two moments, so we aim to construct empirical measures which match the desired target in terms of these moments.
For convenience, we introduce notation for, respectively, matrix mean and deviations,
\begin{align}
    \mmmean{X} \coloneq \frac{1}{N}\sum_{i=1}^N \vv{x}^{(i)}
    \quad\quad
    \mmdev{X} \coloneq \frac{1}{\sqrt{N-1}} \Bigl( \mm{X} - \mmmean{X}\,\vv{1}^\top \Bigr)
     \label{eq:deviation_matrix}
\end{align}
where \(\vv{1}^\top\) is a row vector of \(N\) ones.
The ensemble mean~\eqref{eq:ensemble_mean} and covariance~\eqref{eq:ensemble-covariance} are computed from the empirical measure,
\begin{align}
    \widehat{\Ex}[\vrv{x}]\coloneq \Ex_{\vrv{x}\sim \ELaw[\mm{X}]}[ \vrv{x}] =\widehat{\vv{m}}_{\vrv{x}}
&=\mmmean{X}\label{eq:ensemble_mean} \\
\widehat{\var}(\vrv{x}) \coloneq\var_{\vrv{x}\sim \ELaw[\mm{X}']} (\vrv{x})=\widehat{\mm{C}}_{\vrv{xx}} &\coloneq \frac{1}{N-1} \sum_{i=1}^{N} \left(\vv{x}^{(i)} - \widehat{\vv{m}}_{\vrv{x}}\right)\left(\vv{x}^{(i)} - \widehat{\vv{m}}_{\vrv{x}}\right)^\top  \\
&=\mmdev{X} \mmdev{X}^\top + (\xi^2\mm{I}_{D_{\vrv{x}}}),\label{eq:ensemble-covariance}
\end{align}
where \(\xi^2\) is a scalar constant that introduces a regularisation to the empirical covariance to ensure it is invertible.
Abusing notation, we associate a Gaussian distribution with the empirical measure
\begin{align}
\ELaw[\mm{X}] \approx \Normal(\widehat{\Ex}[\vrv{x}], \widehat{\var}(\vrv{x})) = \Normal(\mmmean{X}, \mmdev{X} \mmdev{X}^\top + \xi^2\mm{I}_{D_{\vrv{x}}}).
\end{align}
For posterior inference about $\vrv{x}$ given $\vrv{y}=\vv{y}^*$, we seek an ensemble $\mm{X}'$ such that~\footnote{see~\cite{LeGland2011Large,Mandel2011Convergence,Kelly2014Wellposedness,Kwiatkowski2015Convergence,DelMoral2017Stability} for finite-ensemble convergence}
\begin{align}
    \Ex_{\vrv{x}\sim \ELaw[\mm{X}']}[ \vrv{x}] &\approx \Ex_{\vrv{x}\sim (\Law[\vrv{x} \gvn \vrv{y}=\vv{y}^*])} [\vrv{x}]\\
    \var_{\vrv{x}\sim \ELaw[\mm{X}']} (\vrv{x}) &\approx \var_{\vrv{x}\sim (\Law[\vrv{x} \gvn \vrv{y}=\vv{y}^*])} (\vrv{x})
\end{align}
We overload the observation operator to apply to ensemble matrices, writing
\begin{align}
    \mm{Y}\coloneq\op{H}\mm{X} &\coloneq \begin{bmatrix}\op{H}(\vv{x}^{(1)}) & \op{H}(\vv{x}^{(2)})& \dots& \op{H}(\vv{x}^{(N)})\end{bmatrix}.
\end{align}
This also induces an empirical joint
\begin{align}
    \ELaw[{\left[\begin{smallmatrix}
        \mm{X}\\
        \mm{Y}
    \end{smallmatrix}\right]}] &\coloneq \Normal\left(\begin{bmatrix}
        \mmmean{X}\\
        \mmmean{Y}
    \end{bmatrix},
    \begin{bmatrix}
        \mmdev{X} \mmdev{X}^\top + \xi^2\mm{I}_{D_{\vrv{x}}} & \mmdev{X} \mmdev{Y}^\top \\
        \mmdev{Y} \mmdev{X}^\top  & \mmdev{Y} \mmdev{Y}^\top + \upsilon^2\mm{I}_{D_{\vrv{y}}}
    \end{bmatrix}
    \right).\label{eq:ensemble-joint}
\end{align}
Here the regularisation term \(\upsilon^2\mm{I}_{D_{\vrv{y}}}\) is introduced to ensure the empirical covariance is invertible, given that the ensemble of observations is typically rank deficient when \(D_{\vrv{y}}>N\).

The Kalman gain in the ensemble setting is constructed by plugging in the empirical ensemble estimates \eqref{eq:ensemble-joint} to \eqref{eq:kalman-gain} obtaining
\begin{align}
\widehat{\mm{K}}= \widehat{\mm{C}}_{\vrv{xx}}\, \op{H}^\top \left(\op{H}\,\widehat{\mm{C}}_{\vrv{xx}}\,\op{H}^\top + \upsilon^2\mm{I}_{D_{\vrv{y}}} + \rho^2\mm{I}_{D_{\vrv{y}}}\right)^{-1}\label{eq:ensemble-kalman-gain}
\end{align}
where \(\rho^2\mm{I}_{D_{\vrv{y}}}\) is the observation error covariance matrix.
The term
\(\upsilon^2\mm{I}_{D_{\vrv{y}}}\) comes from the regularization of the empirical observation covariance (as defined by the ensemble of
\(\mm{Y}\)), while \(\rho^2\mm{I}_{D_{\vrv{y}}}\) represents the known covariance of the observation noise.
Combining these two gives the effective covariance in the observation space.\footnote{The EnKF extends to nonlinear observation operators, though Gaussian assumptions no longer hold exactly~\citep{Evensen2009Data}.}

For compactness we define \(\gamma^2 \coloneq \upsilon^2 + \rho^2\) so that the combined covariance is \(\mm{C}_{\vrv{yy}} = \mmdev{Y} \mmdev{Y}^\top + \gamma^2\mm{I}_{D_{\vrv{y}}}\).
We also define
$\mm{Y}^* = \vv{y}^* \mathbf{1}^\top,$
so that each column of \(\mm{Y}^*\) equals the observation \(\vv{y}^*\). Then, the analysis update for the ensemble is
\begin{align}
    \mm{X}' &= \mm{X} + \widehat{\mm{K}} \left(\mm{Y}^* - \op{H}\mm{X}\right).\label{eq:ensemble-update}
\end{align}
i.e. each ensemble member is updated
\(\vv{x}^{(i)}{}' \gets \vv{x}^{(i)} + \widehat{\mm{K}} \left(\vv{y}^* - \op{H}\,\vv{x}^{(i)}\right).\)
Equating moments, we see that
\(\ELaw[\vrv{x}\sim \mm{X}']
\approx \Law[\vrv{x} \gvn \vrv{y}=\vv{y}^*] \) as desired.
That is, the EnKF analysis equations \eqref{eq:ensemble-kalman-gain} and \eqref{eq:ensemble-update} can be justified in terms of an empirical approximation to the Gaussian posterior update.

\section{Core result}
\begin{proposition}
    The Empirical Matheron Update is equivalent to Ensemble Kalman Update
\end{proposition}
\begin{proof}
Under the substitution
\begin{align}
    \vv{m}_{\vrv{x}} &\to \mmmean{X}, & \mm{C}_{\vrv{xx}} &\to \mmdev{X} \mmdev{X}^\top + \xi^2\mm{I}_{D_{\vrv{x}}},\\
    \vv{m}_{\vrv{y}} &\to \mmmean{Y}, & \mm{C}_{\vrv{yy}} &\to \mmdev{Y} \mmdev{Y}^\top + \gamma^2\mm{I}_{D_{\vrv{y}}},\\
    \mm{C}_{\vrv{xy}} &\to \mmdev{X} \mmdev{Y}^\top, & \mm{C}_{\vrv{yx}} &\to \mmdev{Y} \mmdev{X}^\top,
\end{align}
the Matheron update~\eqref{eq:matheron-update} becomes identical to the ensemble update~\eqref{eq:ensemble-update}.
\end{proof}

\section{Computational Complexity}
The EnKF avoids computing full covariance matrices, instead using empirical ensemble statistics with cost $\mathcal{O}(D_{\vrv{y}}N^2 + N^3 + D_{\vrv{x}}D_{\vrv{y}}N)$. Since ensemble size $N$ is typically much smaller than dimensions $D_{\vrv{x}}, D_{\vrv{y}}$, this is significantly more efficient than the naive Kalman update's $\mathcal{O}(D_{\vrv{y}}^3)$ cost.

\section{Capabilities Unlocked by the Equivalence}\label{sec:capabilities}

\paragraph{(i) Linear-time GP sample paths.}
EnKF algebra lets us compute the Matheron update with $\mathcal O(N^2d)$ cost—or $\mathcal O(Nd)$ with localisation—rather than $\mathcal O(d^{3})$ Cholesky factorizations, enabling multi-million-point GP draws~\citep{Wilson2020Efficiently,Hunt2007Efficient}.

\paragraph{(ii) Streaming and hyper-parameter learning.}
Joint-state augmentation in EnKF permits on-line estimation of GP length-scales and noise variances~\citep{Kuzin2018Ensemble}, giving real-time Bayesian regression as data arrive.

\paragraph{(iii) Differentiable data-assimilation layers.}
Because the affine map is a re-parameterisation, the complete update is compatible with auto-differentiation.  One can embed an EnKF/Matheron block inside a neural network and back-propagate end-to-end, an approach already explored in differentiable simulators~\citep{Spantini2022Coupling,Schillings2017Analysis}.

\paragraph{(iv) Sparse covariance structure via localisation.}
Local ensemble transform Kalman filters (LETKF) taper covariances in physical space, yielding block-banded precisions that transfer verbatim to spatial or graph-based GPs~\citep{Hunt2007Efficient}.

\paragraph{(v) Rigorous error bounds.}
Inverse-problem analyses prove contraction and bias properties of finite-ensemble EnKF~\citep{Schillings2017Analysis}; those proofs now bound the error of path-wise GP samplers for free.

\paragraph{Take-away.}
Treating EnKF analysis as an empirical Matheron  brings mature geophysical tool-kits—localisation, adaptive inflation, square-root variants—into scalable machine-learning while exporting automatic differentiation and sparse GP tricks back to data assimilation.

\section{Numerical Illustration:  1D Kriging by Ensemble Kalman filer}
\label{sec:numerical-demo}

We revisit the classic 1D kriging problem to demonstrate both the direct equivalence of the EnKF update to GP regression and the power of extending this link. We compare three solvers side-by-side: a standard GP implementation, its equivalent EnKF formulation, and the Local Ensemble Transform Kalman Filter (LETKF)~\citep{Bocquet2020Online,Hunt2007Efficient}, a popular data assimilation variant that introduces covariance localization. This illustrates how mature geophysical tools can be immediately repurposed for scalable GP inference. Code is available in the supplement.
We note that the Ensemble Kalman solvers used here are chosen for ease of implementation rather than efficiency of performance, and so we do not expect to see massive performance gains on these small problems.
\begin{quote}
\textbf{Task.}  Recover a latent field
$x\!\in\!\mathbb R^{d}$ on a unit grid from
$m\!=\!d/5$ noisy point samples (scaling with dimension)
$y = Hx + \eta$, $\eta\!\sim\!\mathcal N(0,\tau^{2}I)$ with
$\tau=0.2$.  The prior is
$\mathcal N(0,K)$, $K_{ij}=\sigma^{2}\!\exp(-\lVert r_i-r_j\rVert^{2}/2\ell^{2})$
with $\sigma=1$, $\ell=0.2$.
We test grid sizes $d\in\{200,400,600,800\}$.
\end{quote}

\paragraph{Methods.}
\begin{itemize}
    \item \textbf{GP Regression:} The exact posterior computed via \texttt{GaussianProcessRegressor} from \textsc{scikit-learn}. This serves as our ground truth and scales as $\mathcal{O}(m^3)$ with observations.
    \item \textbf{EnKF:} An implementation of the standard Ensemble Kalman Filter using DAPPER~\citep{Raanes2024DAPPER}. It is the direct, empirical equivalent of the Matheron update and is expected to be computationally cheaper than the exact GP.
    \item \textbf{LETKF:} The Local Ensemble Transform Kalman Filter \citep{Hunt2007Efficient,Bocquet2020Online}, also implemented in DAPPER. This method introduces localization, where each state variable is updated using only a subset of nearby observations. This is a common technique in data assimilation for improving accuracy in high-dimensional systems.
\end{itemize}

We report \textbf{median fit and predict times} obtained with our in-process timer over multiple runs to ensure statistical robustness.

\paragraph{Accuracy.}
\figureref{fig:posterior-samples} shows posterior draws from all three methods using squared-exponential basis functions matched to the generating process (the ``true prior'' setting). We see similar RMSE for the mean estimate between the methods, with a slightly overconfident posterior uncertainty in the ensemble methods.

\paragraph{Speed.}

Figure~\ref{fig:timing-curves} reports fit (training) and predict times as we sweep (top) the number of observations $m$ at fixed state dimension $d$, and (bottom) the state dimension $d$ at fixed $m$.
The exact GP (scikit-learn) shows the expected $\mathcal O(m^3)$ fit scaling.
Its predict time grows with $d$ and $m$—mean-only scales as $\mathcal O(d\,m)$, while computing both mean and standard deviation (our default) scales as $\mathcal O(d\,m^2)$.
In contrast, EnKF and LETKF fits scale roughly linearly in $m$ for fixed ensemble size $N$ (and linearly in $d$ with localization), while prediction is linear in $d$ to read out the analysis mean (and $\mathcal O(Nd)$ if we also materialize an ensemble).
We do not vary $N$ in these plots.
Consequently, as $d$ grows large at fixed $m$, the DA methods yield flatter predict‑vs‑$d$ curves than the exact GP, and as $m$ grows at fixed $d$, DA fits avoid the cubic blow‑up of exact GP regression.

\begin{figure}
  \centering
  \includegraphics[width=\linewidth]{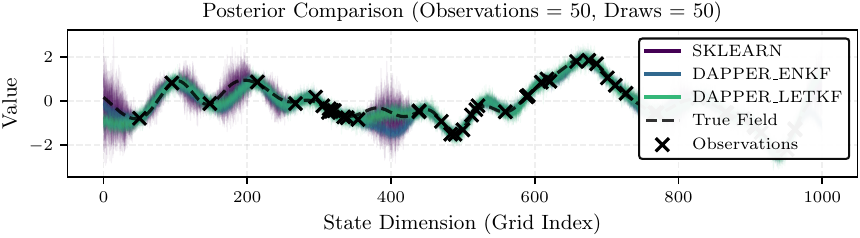}
  \caption{Posterior draws from the exact GP (sklearn), EnKF, and LETKF.}
  \label{fig:posterior-samples}
\end{figure}

\begin{figure}[t]
  \floatconts {fig:timing-curves}%
    {\caption{Fit (solid) vs.\ predict (dashed) wall times.
      Top: scaling with observations \(m\).
      Bottom: scaling with state dimension \(d\).}}%
  {%
    \includegraphics[width=\linewidth]{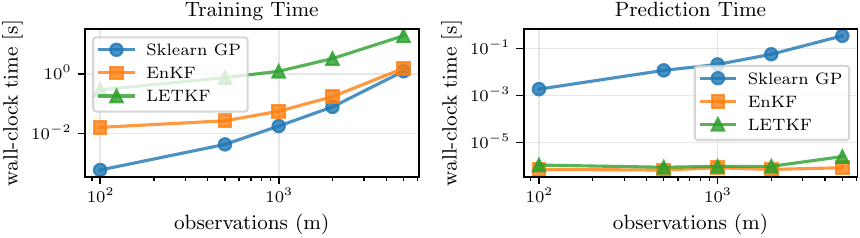}\par\vspace{1ex}
    \includegraphics[width=\linewidth]{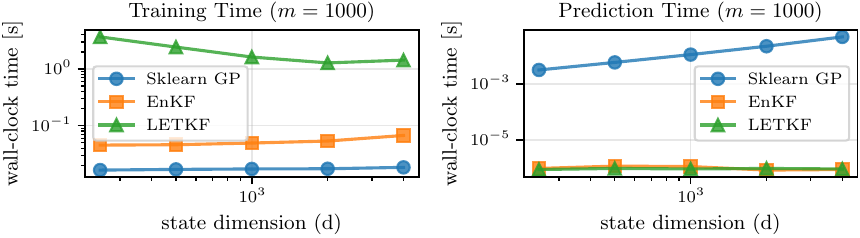}%
  }
\end{figure}

\paragraph{Take-away.}
The results confirm that a standard EnKF is a computationally equivalent method for path-wise GP sampling. More importantly, the equivalence acts as a bridge to a rich ecosystem of data assimilation techniques.
The LETKF demonstrates this: by simply changing one line of code to invoke a localized filter, we gain access to a highly scalable inference method that parallels sparse or domain-decomposition approaches in the GP literature. This opens the door for the GP community to leverage decades of mature, practical DA engineering ``for free.''

\section{Conclusion and Implications}

The affine residual identity is a Rosetta stone linking kriging, optimal-interpolation, ensemble-Kalman, and path-wise GP sampling. This synthesis upgrades each field: geoscientists inherit differentiable, variational, and sparse-GP tools; machine-learning researchers gain decades of localisation, inflation, and convergence theory; and inverse-problem analysts get a unifying algebra that bridges finite-ensemble and infinite-dimensional limits.

The literature of the EnKF is vast, and we have only scratched the surface of the many variants and extensions that have been proposed.
Many of the ideas developed for the EnKF are thus potentially applicable to Gaussian Process regression via the Matheron update.
For example, it provides a clear probabilistic interpretation of the ensemble update that may inspire improved regularization and localization strategies.

By recasting the EnKF update as an empirical Matheron update, we’ve exposed a straightforward, unifying mechanism behind two seemingly different approaches. This connection cuts through the usual complexity and suggests that many of the practical tricks developed for the EnKF—like improved regularization—could be reinterpreted and possibly enhanced when applied to Gaussian process problems. It’s a reminder that sometimes a simple change in perspective can open up entirely new avenues for improvement.

\acks{The author thanks CSIRO's Machine Learning and Artifical Intelligence Future Science Platform for supporting this research.\\
This content was drafted by human hands, with the exception of the experiments and supporting code which was generated by LLMs.}

\bibliography{refs}

\appendix

\section{Measures as densities}\label{sec:densities-please}

In many machine learning and data assimilation applications, probability measures are represented via probability density functions (pdfs) with respect to the Lebesgue measure. For a continuous random variable \(\vrv{x}\) taking values in \(\mathbb{R}^{D_{\vrv{x}}}\), we denote its density by \(p_{\vrv{x}}(\vv{x})\) so that for any measurable set \(A \subset \mathbb{R}^{D_{\vrv{x}}}\),
\begin{equation}
    \Pr\left(\vrv{x} \in A\right)
    = \int_A p_{\vrv{x}}(\vv{x})\,\mathrm{d}\vv{x}.
\end{equation}
The expectation of any measurable function \(f:\mathbb{R}^{D_{\vrv{x}}}\to\mathbb{R}\) is given by
\begin{equation}
    \Ex[f(\vrv{x})] = \int_{\mathbb{R}^{D_{\vrv{x}}}} f(\vv{x})\,p_{\vrv{x}}(\vv{x})\,\mathrm{d}\vv{x}.
\end{equation}

We can write out the results of this paper in terms of densities, but since the primary object of our interest is the empirical measure and the moments of the theoretical and empirical measures, densities are not strictly necessary and result in a rather longer exposition.
Instead, we hope to reassure machine-learners that the two are equivalent in this appendix.

\subsection*{Dirac Functionals and Empirical Measures}

In many practical settings, we work with a finite sample \(\{\vv{x}^{(i)}\}_{i=1}^N\) drawn from the true distribution \(p_{\vrv{x}}(\vv{x})\). The empirical measure associated with this sample is defined as
\begin{equation}
    \ELaw[\mm{X}](A) = \frac{1}{N} \sum_{i=1}^N \mathbf{1}_A(\vv{x}^{(i)}),
\end{equation}
where \(\mathbf{1}_A\) is the indicator function for the set \(A\). When we wish to express the empirical measure in density notation (with respect to the Lebesgue measure), we represent it as a sum of Dirac delta functions:
\begin{equation}
    p_{\text{emp}}(\vv{x}) = \frac{1}{N} \sum_{i=1}^N \delta(\vv{x}-\vv{x}^{(i)}).
\end{equation}

Note that the \emph{Dirac delta} is not a function in the classical sense but a \emph{distribution} (or generalized function). In functional analysis, a Dirac delta centered at \(\vv{x}^{(i)}\) is defined as a linear functional \(\delta_{\vv{x}^{(i)}}\) on a space of test functions (typically smooth functions with compact support) such that
\begin{equation}
    \langle \delta_{\vv{x}^{(i)}}, f \rangle = f(\vv{x}^{(i)}),
\end{equation}
for any test function \(f\). Here, the angle brackets \(\langle \cdot,\cdot \rangle\) denote the action of the distribution on \(f\). This \emph{sifting property} ensures that when integrating a test function against the empirical density, we recover the sample average:
\begin{equation}
    \int_{\mathbb{R}^{D_{\vrv{x}}}} f(\vv{x})\,p_{\text{emp}}(\vv{x})\,\mathrm{d}\vv{x}
    = \frac{1}{N} \sum_{i=1}^N \langle \delta_{\vv{x}^{(i)}}, f \rangle
    = \frac{1}{N} \sum_{i=1}^N f(\vv{x}^{(i)}).
\end{equation}

In this sense, the empirical measure is exactly the sum of Dirac functionals centered at each sample point, and the “density” \(p_{\text{emp}}\) should be understood in the distributional sense. This formulation is particularly useful when comparing the empirical measure to the true measure \(p_{\vrv{x}}\). Although the empirical measure is singular with respect to the Lebesgue measure (since it concentrates mass on a finite set of points), many operations (such as taking expectations of smooth functions) remain well-defined.

\subsection*{From Measures to Densities in Gaussian Conditioning}

For example, suppose the true prior for \(\vrv{x}\) is given by a density \(p_{\vrv{x}}(\vv{x})\) (say, a Gaussian)
\[
    p_{\vrv{x}}(\vv{x}) = \frac{1}{(2\pi)^{D_{\vrv{x}}/2} \,|\mm{C}_{\vrv{xx}}|^{1/2}} \exp\!\left(-\tfrac{1}{2}(\vv{x}-\vv{m}_{\vrv{x}})^\top \mm{C}_{\vrv{xx}}^{-1} (\vv{x}-\vv{m}_{\vrv{x}})\right),
\]
and the likelihood \(p_{\vrv{y}|\vrv{x}}(\vv{y}|\vv{x})\) is similarly specified. Then the joint density is
\begin{equation}
    p_{\vrv{x},\vrv{y}}(\vv{x},\vv{y})
    = p_{\vrv{x}}(\vv{x})\,p_{\vrv{y}|\vrv{x}}(\vv{y}|\vv{x}),
\end{equation}
and the conditional density is given by
\begin{equation}
    p_{\vrv{x}|\vrv{y}}(\vv{x}|\vv{y}^*)
    = \frac{p_{\vrv{x}}(\vv{x})\,p_{\vrv{y}|\vrv{x}}(\vv{y}^*|\vv{x})}{\int_{\mathbb{R}^{D_{\vrv{x}}}} p_{\vrv{x}}(\vv{x}')\,p_{\vrv{y}|\vrv{x}}(\vv{y}^*|\vv{x}')\,\mathrm{d}\vv{x}'}.
\end{equation}

In our work, while we initially denote measures by \(\Law[\cdot]\) (which emphasizes their abstract nature), one can equivalently work with the densities described above.
Thus, updating the ensemble according to the empirical version of the Kalman or Matheron update can be seen as operating directly on these Dirac-based representations.
In practical implementations, one does not manipulate the Dirac deltas directly; instead, one works with the sample values and their empirical moments, which—as shown earlier—are equivalent (in the limit of large \(N\)) to the full probabilistic update based on the density.

This density-based view, incorporating Dirac functionals to represent empirical measures, is particularly common in machine learning, where it provides an intuitive bridge between abstract measure-theoretic probability and the concrete computations performed with finite samples.

The need to throw around integrals, introduce and worry about density normalisation and so on, we find a little distracting and thus it is not used in the main text.

\end{document}